\documentclass[letterpaper, 10 pt, conference]{ieeeconf}  %

\IEEEoverridecommandlockouts                               %

\usepackage{svg} %
\usepackage{graphicx}
\usepackage{graphics} %
\usepackage{bbm}

\usepackage{amsmath} %
\usepackage{amssymb}  %
\usepackage{color}
\usepackage{xspace}
\usepackage{booktabs}
\usepackage{multirow} 
\usepackage{graphicx,subcaption,lipsum}
\usepackage[
    style=ieee,
    natbib=true,
    citestyle=numeric-comp,
    doi=false,
    isbn=false,
    url=false]{biblatex} 

\makeatletter
\let\NAT@parse\undefined
\makeatother

\usepackage{hyperref}
\usepackage{cleveref}

\definecolor{mydarkblue}{rgb}{0,0.08,0.45}
\definecolor{mydarkgreen}{RGB}{0, 139, 69}
\definecolor{mygreen2}{RGB}{0 205 0}
\definecolor{mybrown}{RGB}{139 69 19}
\hypersetup{
	colorlinks=true,
	linkcolor=blue,
	urlcolor=magenta,
	citecolor=mygreen2,
}

\newcommand{\method}{SafeDPA\xspace}
\newtheorem{theorem}{Theorem}

\newtheorem{asm}{Assumption}

\Crefname{asm}{Assumption}{Assumption}

\newcommand{\R}[0]{\mathbb{R}\xspace}
\newcommand{\Fix}[0]{$\text{Fix-}\alpha$\xspace}
\newcommand{\RMAP}[0]{$\text{RMA-P}\beta$\xspace}

\title{\LARGE \bf
Safe Deep Policy Adaptation
}

\author{Wenli Xiao$^{*}\thanks{* These authors contributed equally to this work.}$, Tairan He$^{*}$, John Dolan, and Guanya Shi 
\thanks{The authors are with the Robotics Institute, Carnegie Mellon University, USA. {\tt\small\{wxiao2, tairanh, jdolan, guanyas\}@andrew.cmu.edu}.}
\thanks{Paper website: \href{https://sites.google.com/view/safe-deep-policy-adaptation}{https://sites.google.com/view/safe-deep-policy-adaptation}}
}

\begin{document}

\maketitle
\thispagestyle{empty}
\pagestyle{empty}

\begin{abstract}
A critical goal of autonomy and artificial intelligence is enabling autonomous robots to rapidly adapt in dynamic and uncertain environments.
Classic adaptive control and safe control provide stability and safety guarantees but are limited to specific system classes. In contrast, policy adaptation based on reinforcement learning (RL) offers versatility and generalizability but presents safety and robustness challenges. 
We propose \method, a novel RL and control framework that simultaneously tackles the problems of policy adaptation and safe reinforcement learning. \method jointly learns adaptive policy and dynamics models in simulation, predicts environment configurations, and fine-tunes dynamics models with few-shot real-world data. A safety filter based on the Control Barrier Function (CBF) on top of the RL policy is introduced to ensure safety during real-world deployment.
We provide theoretical safety guarantees of \method and show the robustness of \method against learning errors and extra perturbations.
Comprehensive experiments on (1) classic control problems
(Inverted Pendulum), (2) simulation benchmarks (Safety
Gym), and (3) a real-world agile robotics platform (RC
Car) demonstrate great superiority of \method in both safety and task performance, over state-of-the-art baselines. 
Particularly, \method demonstrates notable generalizability, achieving a 300\% increase in safety rate compared to the baselines, under unseen disturbances in real-world experiments.
\end{abstract}

\section{INTRODUCTION}
Adaptation is a crucial aspect of autonomous systems, particularly when confronted with dynamic and uncertain environments~\cite{aastrom2013adaptive,slotine1987adaptive,lavretsky2012robust,o2022neural,huang2023datt,hanover2021performance,kumar2021rma}. In these systems, robots must continually adjust their behavior and decision-making strategies to effectively respond to changing conditions.
Adaptation in the context of control systems, often referred to as \textit{adaptive control}~\cite{aastrom2013adaptive}, is a fundamental concept in control theory.
Classic adaptive control approaches the adaptation problem by continually adjusting the control system parameters to maintain or improve its performance in the presence of changing dynamics, uncertainties, or disturbances~\cite{parks1966liapunov,khosla1985parameter,ioannou1996robust}.
However, the applicability of classic adaptive control methods is often constrained by their dependence on analytical representations of system dynamics and specific system classes (e.g., systems with linear parametric uncertainties). 
These limitations restrict their utility in more complex systems.
In the field of \textit{reinforcement learning} (RL)~\cite{kaelbling1996reinforcement}, policy adaptation operates in a more general setting, often 
 without explicit system models, but it shares the core concept of adaptation~\cite{zhao2020sim}.
Successful policy adaptation~\cite{kumar2021rma,lee2020learning,tan2018sim,hwangbo2019learning,peng2020learning,huang2023datt} empowers autonomous systems to excel in a wide range of tasks and environments, making them more versatile and capable of handling real-world challenges effectively.

Though RL policy adaptation has achieved remarkable progress in autonomous systems, safety remains a critical hurdle that limits its application in real-world control tasks, as any system must ensure that its actions do not lead to harmful consequences for itself or its surroundings.
Existing end-to-end safe RL methods~\cite{liang2018accelerated,tessler2018reward,achiam2017constrained} 
are hampered in real-world deployment due to the lack of safety guarantees.
Hierarchical safe RL methods that provide theoretical safety guarantees~\cite{zhao2021model,cheng2019end} are susceptible to breaking when confronted with modest changes in dynamics or disturbances. 
To the best of our knowledge, there is no prior work that jointly tackles policy adaptation and safe reinforcement learning with safety guarantees.

In this paper, we propose \method, a general RL and control framework that enables rapid policy adaptation in dynamic environments with safety guarantees. \method (as illustrated in \Cref{fig:method-framework}) first learns a policy and a control-affine dynamics model, both conditioned on environment configurations, in simulation. Subsequently, \method learns an adaption module that predicts the environment configuration using historical data in simulation. 
Finally, in real-world deployments, \method first fine-tunes the pre-trained dynamics and adaption module, using scant few-shot data. \method then deploys the adaptive policy with a safety filter, based on the Control Barrier Function (CBF), ensuring safety is upheld. Our contributions are:
\begin{itemize}
    \item To the best of our knowledge, \method is the first framework that jointly tackles policy adaptation and safe reinforcement learning. 
    \item Under mild assumptions, we provide theoretical safety guarantees on \method. Further, we show the robustness of \method against learning errors and extra perturbations, which also motivates and guides the fine-tuning phase of \method.
    \item We conduct comprehensive experiments to evaluate \method and state-of-the-art baseline methods in (1) a classic control problem \textit{(Inverted Pendulum)}, (2) safe RL simulation benchmarks \textit{(Safety Gym)}, and (3) a real-world agile robotics platform \textit{(RC Car)}.
    Particularly, \method demonstrates notable generalizability, achieving a 300\% increase in safety rate compared to the baselines, under unseen disturbances in real-world experiments, as illustrated in \Cref{fig:RC-Car-Demo} and \Cref{TABLE:Real-Car}.
\end{itemize}

\begin{figure*}[t]
    \centering
    \includegraphics[width=0.9\textwidth]{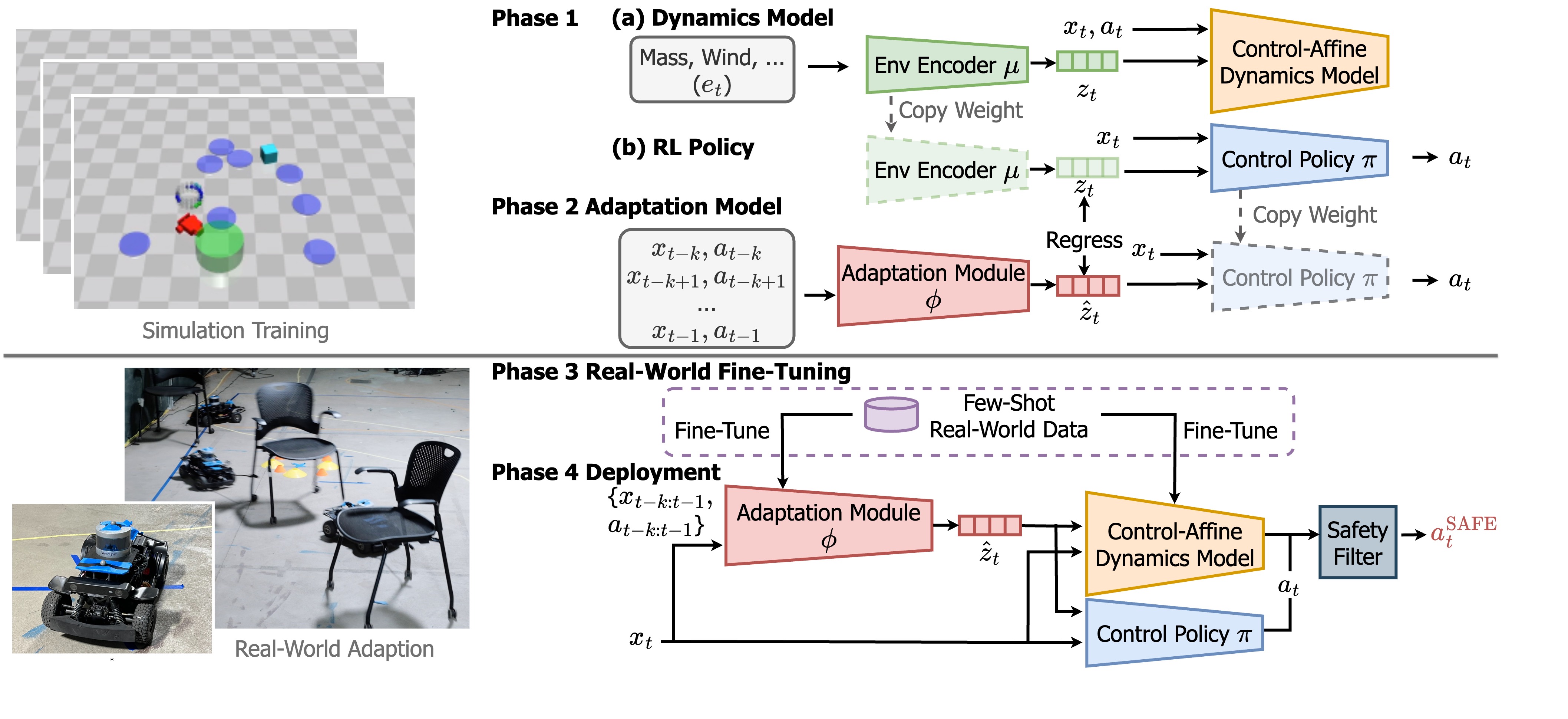}
    \caption{Overview of the four phases of \method. In \textbf{Phase 1 (a)}, the environment encoder $\mu_{\theta_\mu}$ and dynamics model $f_{\theta_f}$, $g_{\theta_g}$ are jointly trained with offline dataset collected by a random policy in simulation. In \textbf{Phase 1 (b)}, we make the parameters of environments encoder $\mu_{\theta_\mu}$ frozen, and the base policy is trained in simulation using model-free RL. In \textbf{Phase 2}, we train the adaption module $\phi_{\theta}(x_{t-k:t-1}, a_{t-k:t-1})$ to fit environments encoder $\mu_{\theta_\mu}$ with the history of state and actions with on-policy data. In \textbf{Phase 3}, we fine-tune our learned dynamics model trained in simulation with few-shot real-world data. In \textbf{Phase 4}, we leverage the learned adaptive dynamics to construct a CBF-based safety filter on top of the adaptive RL policy to ensure safety during real-world deployment. 
    }
    \label{fig:method-framework}
\end{figure*}

\section{RELATED WORK}
\subsection{Safe Reinforcement Learning}
Safe RL methods can be divided into two categories~\cite{zhao2023state}: 1) \textit{end-to-end methods} and 2) \textit{hierarchical methods}. 
One representative of \textit{end-to-end methods} is Lagrangian methods that solve a primal-dual optimization problem to satisfy safety constraints~\cite{liang2018accelerated, tessler2018reward}.
Another branch of \textit{end-to-end methods} is direct policy optimization. Constrained policy optimization (CPO)~\cite{achiam2017constrained} extends a standard trust-region RL algorithm to CMDP, leveraging a novel bound that relates the expected cost of two policies to their average divergence.
\textit{Hierarchical methods} solve the constrained learning problem by adding a safety filter upon the RL policy network, projecting unsafe actions to safe actions. Safety Layer~\cite{dalal2018safelayer} solves a QP problem of cost functions with learned dynamics to safeguard unsafe actions, but it does not provide adaptation ability and safety guarantees. ISSA~\cite{zhao2021model} proposes an efficient sampling algorithm for general control barrier functions to get safe actions and guarantees zero violation during agent learning, but it requires a perfect dynamics model for a one-step prediction. Our method can be categorized in the \textit{hierarchical methods} but with fewer assumptions on safety certificate functions and dynamics.
\subsection{Safe Control}
The safe control community has built a rich body of literature studying how to persistently satisfy a hard safety constraint. The most widely used methods are \textit{energy-based methods} including potential field methods~\cite{khatib1986real}, control barrier functions~\cite{ames2014control}, safe set algorithms~\cite{liu2014control}, and sliding mode algorithms~\cite{gracia2013reactive}. The method most related to our work is control barrier functions. Control barrier functions give a sufficient approach to guarantee forward invariance of the safe set under the respective control law, and can be naturally used as a safety filter upon RL policy networks.

\subsection{Adaptive Control and Policy Adaptation}
Classic adaptive control algorithms often investigate the problem of stability guarantees (e.g., Lyapunov stability) and tracking error convergence (e.g., asymptotic convergence)~\cite{slotine1991applied, aastrom2013adaptive}.
Adaptive control via online system identification is critical to deploying robotics in the real world~\cite{slotine1987adaptive}. However, inferring all the exact physics parameters can be unnecessary and difficult~\cite{kumar2021rma}. 
Recently, there has been increased interest in studying online policy adaptation of reinforcement learning for robotics~\cite{kumar2021rma,lee2020learning,tan2018sim,hwangbo2019learning,peng2020learning}.
These methods train RL controllers in simulation with environmental parameters (e.g., terrains and friction parameters) and then deploy the policy on real-world robots using adaptation techniques. 
There are also many Meta-RL works~\cite{luo2021mesa,nagabandi2018learning,he2022reinforcement} focusing on improving the task-wise adaptation ability.
Among these policy adaptation works, the most relevant work, Rapid Motor adaptation (RMA)~\cite{kumar2021rma}, proposes a teacher-student learning framework
consisting of two components: a base policy and an adaptation module. The combination of these components and the student-teacher training scheme enables the legged robots to adapt to unseen terrains. 
Compared to RMA, our approach has safety guarantees and better bridges the sim-to-real gap (as shown in \Cref{SECTION:EXPERIMENT}) by fine-tuning with few-shot real-world data.
To the best of our knowledge, our work is the first to investigate the policy adaptation of safe RL in dynamics.

\begin{figure*}[t]
    \centering
    \includegraphics[width=0.9\textwidth]{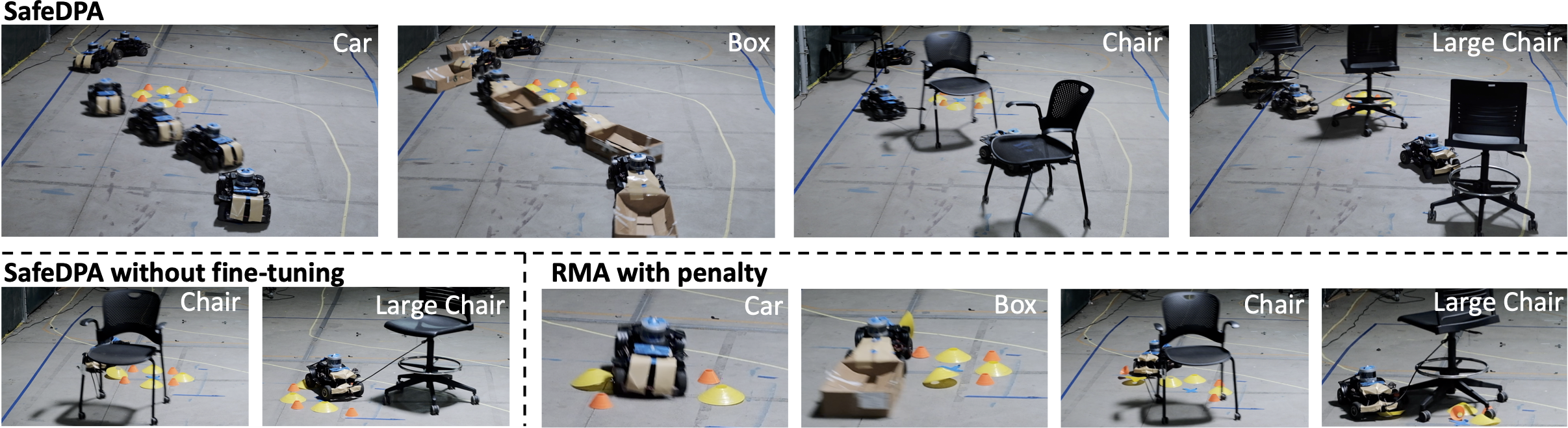} 
    \caption{Comparison of \method, \method without fine-tuning, and RMA with penalty on RC Car platforms. We showcase the successful trajectory of \method in four tasks, alongside instances of an unsafe event (i.e., collision) for \method without fine-tuning and RMA with penalty.
    \method safely achieves the goal, although in training or fine-tuning it never sees the box or chairs.
    This highlights the exceptional generalizability and adaptability of \method. 
    We present video demonstrations in \href{https://sites.google.com/view/safe-deep-policy-adaptation}{https://sites.google.com/view/safe-deep-policy-adaptation}.
    }
    \label{fig:RC-Car-Demo}
\end{figure*}

\section{Problem Statement}
\label{SECTION:Problem-Statement}
In this work, we consider a system with control-affine dynamics, which could generally represent any rigid-body robotic system~\cite{enqvist2005linear}, as:
\begin{equation}
    x_{t+1} = f(x_t,e_t) + g(x_t,e_t)a_t,
\label{eq:dynamics_w_et}
\end{equation}
with time step $t$, state $x_t\in \R^n$, action $a_t\in \R^m$, environment configuration $e_t\in \R^k$, unactuated dynamics $f: \R^m\times\R^k \rightarrow \R^m$, and actuated dynamics $g:\R^m\times\R^k \rightarrow \R^{n \times m}$.
The environment configuration $e_t$ (unknown in real-world deployment) represents physical properties that may change in the system, e.g., friction, payload, disturbance, etc. Different types of configuration $e_t$ may exert distinct influences on the system dynamics. 
For instance, in the case of RC Car (as shown in \Cref{fig:RC-Car-Demo}), dragging a box or a heavy chair can affect both $f$ and $g$. However, certain disturbances only impact $f$, such as an additional wind force. These environment configurations $e_t$ play a crucial role in adaptive control. Changes in wind or dragging can cause drones to crash~\cite{liu2020robust} or force cars to collide with obstacles~\cite{cheng2019end}, resulting in system breakdown, especially for safety-critical systems. 
On the other hand, $e_t$ also has huge impacts on control policy performance. Therefore, our objective is to design a framework that accounts for varying environment configurations, ensuring tasks are executed with a high probability of safety guarantee while achieving high task rewards. We formulate this problem as: $\max_{\pi}  \mathbb{E}_{\tau \sim \pi} \left[\sum_{t=0}^{\infty} \gamma^t r(x_t)\right] \text{s.t. }  x_t \in \mathcal{C}$.
The safe set $\mathcal{C}$ defines a region in the state space that the robot should always stay in. We apply the standard RL~\cite{schulman2017proximal,kumar2021rma,schulman2015trust} setting for solving that, say, $\pi$ is the control policy, $r \in \R$ denotes reward, and $\gamma$ is the discounted factor. 

For the above problem, it is nontrivial to model the dynamics of the robot system conditioned on environment configurations $e_t$, because though it is easy to access the environment configuration $e_t$ in simulation, it is difficult to obtain it during real-world deployments. Moreover, the sim-to-real gap degrading the task performance might be acceptable in certain applications, but this gap might lead to unacceptable catastrophes in safety-critical scenarios. This makes fine-tuning the adaptive dynamics model using few-shot real-world data necessary. In the following section, we will introduce how we tackle these challenges.

\section{Safe Deep Policy Adaptation}
\label{SECTION:METHOD}

In this section, we elaborate on our proposed \method, as outlined in \Cref{fig:method-framework}. \method comprises four phases: Training dynamics model and policy in simulation, training adaptation module in simulation, fine-tuning with few-shot real-world data, and deployment with safe filter, each of which is elaborated upon in the subsequent subsections.

\subsection{Phase 1: Dynamics Model Learning in Simulation} 
Our first step is to train a neural control-affine dynamics model $f_{\theta_f}$ and $g_{\theta_g}$ to approximate system dynamics $f$ and $g$ in simulation. We apply an environment encoder $\mu_{\theta}$ to extract latent representation $z_t$ of environment configuration $e_t$, then pass $z_t$ to dynamics model $f_{\theta_f}$ and $g_{\theta_g}$. We implement $\mu_{\theta_\mu}$, $f_{\theta_f}$, and $g_{\theta_g}$ as multi-layer perceptrons (MLP) and train these models jointly through end-to-end supervised learning. The training loss is defined by:
\begin{equation}
\resizebox{\hsize}{!}{$
\begin{aligned}
\mathcal{L}^{\text{DYN}}_{\theta_f, \theta_g, \theta_\mu} &= \frac{1}{|\mathcal{D}|}\sum\limits_t \|x_{t+1} - \hat{x}_{t+1}\|^2_2 \\
&= \frac{1}{|\mathcal{D}|}\sum\limits_t\| x_{t+1} - f_{\theta_f}(x_t, \mu_{\theta}(e_t)) - g_{\theta_g}(x_t, \mu_{\theta_\mu}(e_t)) a_t\|^2_2,
\end{aligned}
$}
\label{EQ:Dynamics-Loss}
\end{equation}
    where $\mathcal{D}$ is a dataset of transitions $(x_{t}, e_t, a_{t}, x_{t+1})$. We synthesize dataset $\mathcal{D}$ by rolling out a random walking agent in the simulation. We randomize the environment configuration $e_t$ at every step to collect diverse data. 

With the learned encoder $\mu_{\theta_\mu}$, we can train an adaptive RL\footnote{While we utilize an RL agent as the control policy, our proposed \method is controller-agnostic. The RL controller can be substituted with other methods, such as Model Predictive Control (MPC).} policy $\pi_{\theta_\pi}(x_t, \mu_\theta(e_t))$ by integrating model-free RL algorithms (e.g. PPO~\cite{schulman2017proximal}) with the parameter-frozen environment encoder $\mu_{\theta_\mu}$. RL will learn the policy that maximizes expected return $\mathbb{E}_{\tau \sim \pi} \left[\Sigma_{t=0}^{\infty} \gamma^t r(x_t)\right]$ under different environment latent $z_t$. 
In our early empirical experiments, we find that the training scheme in RMA~\cite{kumar2021rma} (i.e., jointly training the environment encoder $\mu_{\theta_\mu}$ with the RL agent) does not apply for dynamics learning since
the latent variable $z_t$ learned from RL loss
does not provide sufficient context for accurate dynamics prediction. 
Therefore, our approach largely improves the dynamics prediction by pertaining the environment encoder only with $\mathcal{L}^{\text{DYN}}_{\theta_f, \theta_g, \theta_\mu}$ to extract a more meaningful representation for dynamics learning.

\subsection{Phase 2: Train Adaptation Module in Simulation}

As mentioned in \Cref{SECTION:Problem-Statement}, the environment configurations $e_t$ are often inaccessible in the real world. Inspired by RMA~\cite{kumar2021rma} and UP-OSI~\cite{yu2017preparing}, we learn an adaptation module $\phi_{\theta}(x_{t-k:t-1}, a_{t-k:t-1})$ that continuously predicts the environment latent $\hat z_t$ for dynamics model and control policy, where $\{x_{t-k:t-1}, a_{t-k:t-1}\}$ is a history of state-action pairs with length $k$. We predict latent $z_t$ rather than environment configurations $e_t$, since it is more generalizable and compact to consider the impact of $e_t$ on dynamics than its actual value~\cite{kumar2021rma}. We model the adaptation module using a 1-D CNN to capture spatiotemporal correlation and train it via an iterative procedure in simulation. For each iteration, we roll out an RL agent in simulation with random episodic $e_t$ to collect dataset $\mathcal{D}^{\text{ADAPT}}$ with tuples $((x_{t-k:t-1}, a_{t-k:t-1}), e_{t-1})$. We then train $\phi_{\theta_\phi}$ via supervised learning with $\mathcal{D}^{\text{ADAPT}}$ to minimize loss $\mathcal{L}^{\text{ADAPT}}_{\theta_\phi}=\frac{1}{|\mathcal{D}^{\text{ADAPT}}|}\sum_t||z_t - \hat z_t||^2_2$, where $z_t = \mu_{\theta_\mu}(e_t)$ and $\hat{z}_t = \phi_{\theta_\phi} (x_{t-k:t-1}, a_{t-k:t-1})$.

\subsection{Phase 3: Real-World Data Fine-Tuning}

Since the safety performance of \method heavily relies on an accurate adaption module and dynamics model (as discussed in \Cref{SECTION:Theory}), the safety can be largely affected by the sim-to-real gap. 
Leveraging the adaptive structure of the dynamics model in \method, we are able to align simulation with the real world by fine-tuning the dynamics model and adaptation module with few-shot (compared with the training set) real-world data. We propose to fine-tune models via end-to-end supervised learning, which minimize
$\mathcal{L}^\text{TUNE}_{\theta_f, \theta_g, \theta_\phi}=\frac{1}{|D_\text{real}|}\sum_{i}||\hat x_{t+1} - x_{t+1}||_2^2,$   
where $\hat x_{t+1}=f_{\theta_f}(x_t,\hat z_t)+g_{\theta_g}(x_t,\hat z_t)a_t$, $\hat z_t = \phi_{\theta_\phi}(x_{t-1:t-k},a_{t-1:t-k})$. Since the dynamics model is independent of control policy, the dataset $D_\text{real}$ can be collected from various sources, such as teleoperation, random walks, etc. 
The fine-tuned neural networks $f_{\theta_f}, g_{\theta_g}, \phi_{\theta_\phi}$ have shown empirical success and necessity in real-world applications, as evidenced in \Cref{SECTION:EXPERIMENT-RC-CAR}.
Remarkably, \method displays strong generalizability to adapt to new disturbances like \textbf{Chairs} and \textbf{Box} (refer to \Cref{fig:RC-Car-Demo}), even though it has never seen these objects in both simulation and fine-tuning. 

\subsection{Phase 4: Safe Filter and Real-World Deployment}

To achieve safe real-world deployment, we combine the adaptation module and dynamics model and leverage the control barrier function (CBF)~\cite{ames2014control} to construct a safe filter, which can guarantee the robot always stays in the safe set $\mathcal{C}$ (i.e., forward invariance). We first construct a CBF $h:\R^n\rightarrow \R$, which satisfies: %
\begin{equation}
    \mathcal{C}^h :=\{x\in \R^n: h(x) \geq 0\} \subseteq \mathcal{C},\forall x \in \partial\mathcal{C}^h, \frac{\partial h}{\partial x} \neq 0,
\end{equation}
\begin{equation}
    \forall x\in\partial \mathcal{C}^h, h(x_{k+1})\geq (1-\eta) h(x_k), 0 < \eta \leq 1,
    \label{EQUATION:DTCBF}
\end{equation}
where $\mathcal{C}^h$ is a subset of the safe set $\mathcal{C}$, and $\eta$ is a scalar representing the conservativeness of CBF. If \Cref{EQUATION:DTCBF} is satisfied for all $x_t \in \mathcal{C}^h$, then $\mathcal{C}^h$ and the safe set $\mathcal{C}$ are rendered forward invariant~\cite{agrawal2017discrete, ames2016control}. We adopt affine barrier functions as \cite{cheng2019end}: $h=p^Tx+q$, ($p\in\mathbb{R}^n,q\in\mathbb{R}$), for the simplicity of computation. We then formulate an optimization problem (CBF-QP) satisfying \Cref{EQUATION:DTCBF} that can be solved at each time step:
\begin{equation*}
\resizebox{0.45\textwidth}{!}{$
\begin{array}{cl}
a^{\text{SAFE}}_t =\underset{a_t}{\operatorname{argmin}} & \left\|a_t - \pi_{\theta_\pi}(x_t, \hat z_t)\right\|_2 \\
\text { s.t. } & p^T f_{\theta_f}\left(x_t,\hat z_t\right)+p^T g_{\theta_g}\left(x_t,\hat z_t\right) a_t+p^T q \\
& \geq (1-\eta) h\left(x_t\right)+\epsilon \\
& a_t \in \mathcal{A},
\end{array}
$}
\label{eq:cbf-qp}
\end{equation*}
where $\mathcal{A}$ is feasible action set, $\hat{z}_t = \phi_{\theta_\phi} (x_{t-k:t-1}, a_{t-k:t-1})$, and $\epsilon$ is a variable of robust margin for safety which will be quantified in the following theory part (\Cref{SECTION:Theory}). This QP seeks to modify RL policy to meet safety constraints with minimal interference. \method takes the output of safety filter $a^{\text{SAFE}}_t$ to execute in the real world. 
The theoretical analysis on the safety of $a^{\text{SAFE}}_t$ is presented in \Cref{SECTION:Theory}.

\section{Theoretical Analysis}
\label{SECTION:Theory}
In this section, we begin to present the safety guarantee of \method by introducing an assumption on the dynamics prediction error.
\begin{asm}[Bounded Prediction Error]
$\forall x_t, a_t, e_t$, $\| f(x_t, e_t) - f_{\theta_f}(x_t, e_t) \|_1 < \epsilon_f$, $\| g(x_t, e_t) - g_{\theta_g}(x_t, e_t) \|_1 < \epsilon_g$, $\|\hat{z_t} - z_t\|_1 = ||\phi_{\theta_\phi}(x_{t-k:t-1},a_{t-k:t-1})-\mu(e_t)|| < \epsilon_z$.
\label{asm:error}
\end{asm}
We further introduce another assumption regarding the Lipschitz continuity of the (neural) dynamics. These assumptions are necessary because they provide critical quantifications into the properties of dynamics and neural networks. Otherwise, we would lack the tools to rigorously analyze the reliability of these networks~\cite{berkenkamp2017safembrl}.
\begin{asm}[Lipschitz Continuity]
The dynamics $f( \cdot )$, $g(\cdot)$ in \Cref{eq:dynamics_w_et} are $L_f$, $L_g$ Lipschitz continuous with respect to the 1-norm. The learned dynamics $f_{\theta_f}( \cdot )$, $g_{\theta_g}(\cdot)$ in \Cref{EQ:Dynamics-Loss} are $L_{f_{\theta_f}}$, $L_{g_{\theta_g}}$ Lipschitz continuous with respect to the 1-norm. %
\label{asm:continuity}
\end{asm}
\begin{theorem}[Safe Control]
Under \Cref{asm:error,asm:continuity}, then solving the safety condition $p^T f\left(x_t\right)+p^T g\left(x_t\right) a_t+p^T q \geq (1-\eta) h\left(x_t\right)-\epsilon $ in \Cref{eq:cbf-qp} will guarantee the forward invariance of the safe set $\mathcal{C}$ (i.e., $x_{t+n} \in \mathcal{C}, \forall n=1,2,3\cdots$) if $\epsilon > \epsilon_f p^T \mathbf{1} + \epsilon_g \|a\|_{1}^{max} p^T \mathbf{1} + \epsilon_z (L_f + L_{f_{\theta_f}}) + \epsilon_z \|a\|_{1}^{max} (L_g + L_{g_{\theta_g}})$ where we denote $\|a\|_{1}^{max} = \max_{\forall a_t \in \mathcal{A}} \|a_t\|_1$.
\end{theorem}
\begin{proof}
We consider the worst-case prediction of the $h(\hat{x}_{t+1})$, which is $\min_{x_t, a_t} h(\hat{x}_{t+1})$ where $\hat{x}_{t+1} = f_{\theta_f} (x_t, \hat{z_t}) + g_{\theta_g} (x_t, \hat{z_t})a_t$.
The prediction error of $h(\hat{x}_{t+1})$ comes from the prediction error of $\hat{x}_{t+1}$, which can be further decoupled into two parts: 1) the approximation error of $f_{\theta_f}$ and $g_{\theta_g}$; 2) the approximation error of $\phi_{\theta_\phi}$. 
As for 1) the approximation error of $f_{\theta_f}$ and $g_{\theta_g}$, supposing we have perfect regression of $z_t$ (i.e., $\hat{z_t} = z_t$), denote the maximum error of $h(\hat{x}_{t+1})$ under this situation as $\Delta_1 = h(x_{t+1}) - h(f_{\theta_f} (x_t, \hat{z_t}) + g_{\theta_g} (x_t, \hat{z_t})a_t) <   \epsilon_f p^T \mathbf{1} + \epsilon_g \|a\|_{1}^{max} p^T \mathbf{1}$. 
As for 2) the approximation error of $\phi_{\theta_\phi}$, denote the maximum error of $h(\hat{x}_{t+1})$ under this situation as $\Delta_2 = \|\hat{z_t} - z_t\|^{max}_1 (L_f + L_{f_{\theta_f}}) + \|\hat{z_t} - z_t\|^{max}_1 \|a\|^{max}_{1} (L_g + L_{g_{\theta_g}}) < \epsilon_z (L_f + L_{f_{\theta_f}}) + \epsilon_z \|a\|_{1}^{max} (L_g + L_{g_{\theta_g}})$.
Combining $\Delta_1$ and $\Delta_2$, by setting $\epsilon > \Delta_1 + \Delta_2$ and solving $h(\hat{x}_{t+1}) \geq (1-\eta)h(x_t) - \epsilon$, we have $h(x_{t+1}) \geq (1-\eta)h(x_t)$. The forward invariance of the safe set $\mathcal{C}$ is a natural proof following ~\cite{ames2019control}.
\end{proof}

Note that the value of $\epsilon$ depends on two types of error: 1) dynamics learning error (i.e., $\epsilon_f$ and $\epsilon_g$); and 2) adaptation regression error (i.e., $\epsilon_z$).
As these two errors decrease (e.g., with better prediction models or fine-tuning with real-world data), $\epsilon$ will decrease as well.
In practice, one can statistically estimate $\epsilon$ using fine-tuning data, but quantifying exact error bounds is a challenging machine learning problem~\cite{berkenkamp2017safembrl,liu2020robust}.

\begin{figure}[htbp]
    \centering
    \includegraphics[width=0.8\columnwidth]{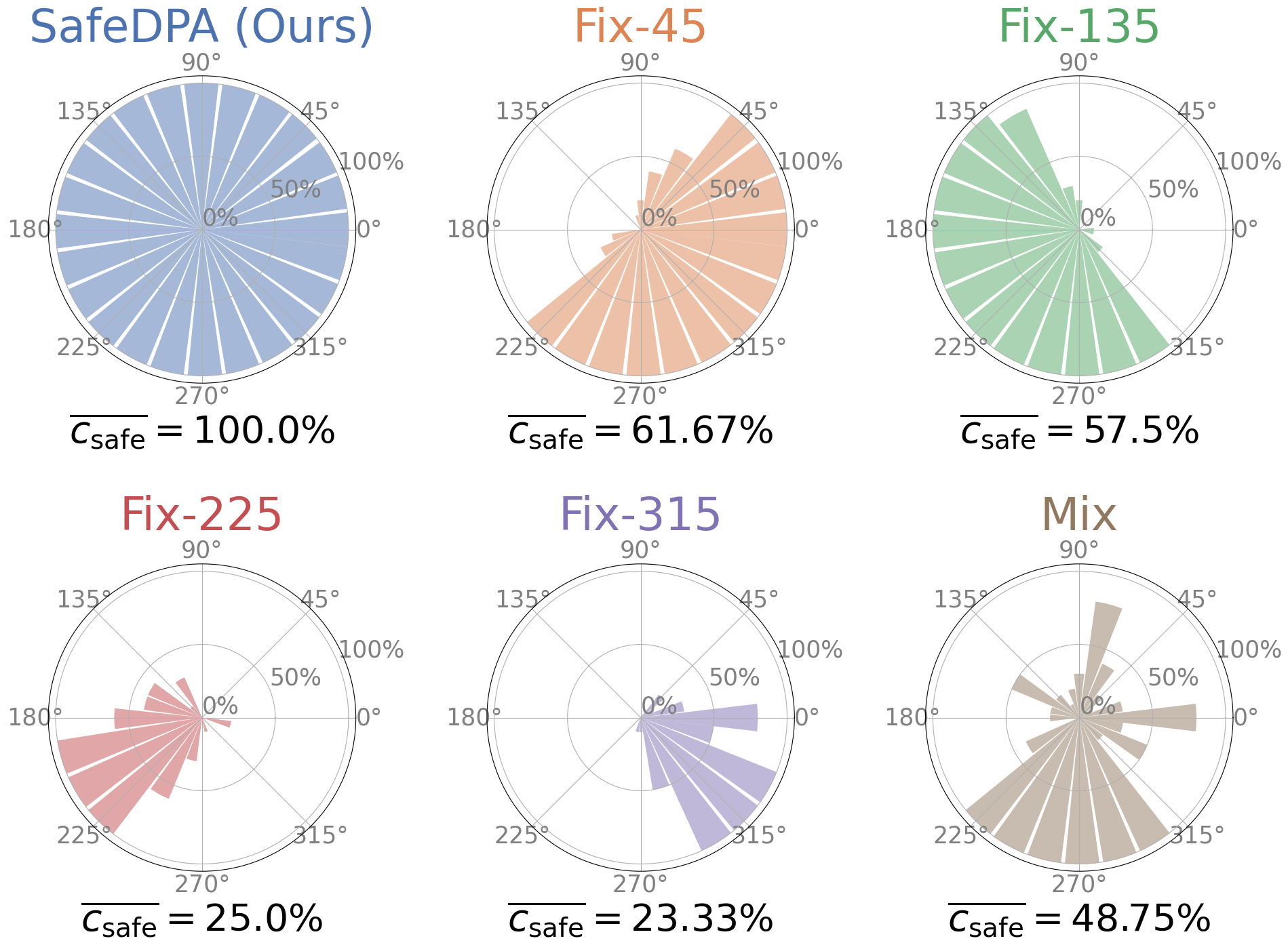}
    \caption{The area of the colored region represents the safety rate. In the inverted pendulum task, \method consistently achieves the highest safety rate with zero violations across all directions, surpassing \Fix and Mix where these baselines only maintain high safety rates in specific directions.}
    \label{FIG:Result-Pend-Safe}
\end{figure}

\section{EXPERIMENTS}
\label{SECTION:EXPERIMENT}

In this section, we demonstrate extensive experimental results in Inverted Pendulum, Safety Gym~\cite{Ray2019}, and Real-World RC Car, to address the following questions: 
\begin{itemize}
\item \textbf{Q1}: Can \method guarantee safety in diverse tasks and various environment configurations?
\item \textbf{Q2}: Can \method strike the perfect balance between safety and task performance? How does \method compare with state-of-the-art safe RL methods?
\item \textbf{Q3}: How well does \method generalize to unseen scenarios in the real-world? How does the fine-tuning phase in \method boost performance?
\end{itemize}

\begin{figure}[htbp]
    \centering
    \includegraphics[width=0.9\columnwidth]{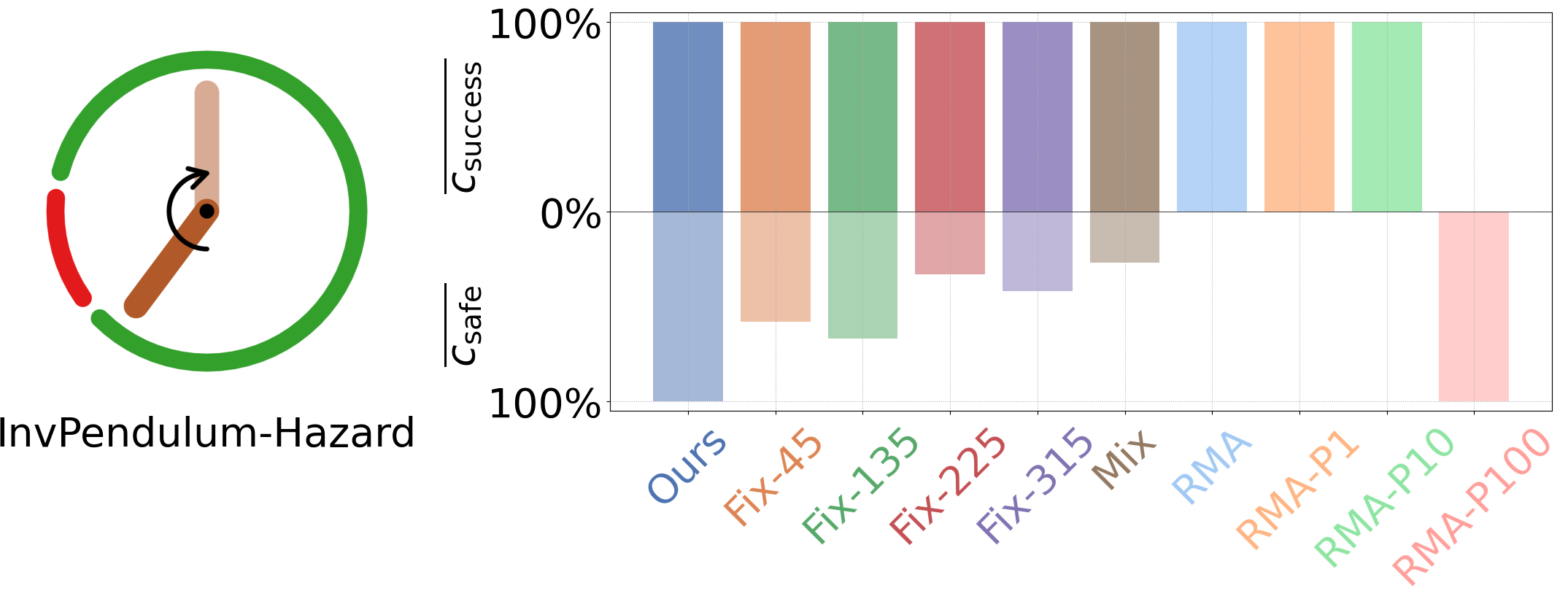}
    \caption{On the left, we show the InvPendulum-hazard environment. On the right, we demonstrate the success rate and safety rate of \method and baselines where \method is the only algorithm that achieves both 100\% for success rate and safety rate. 
    }
    \label{FIG:Result-Pend-hazard}
\end{figure}

\begin{figure*}[htbp]
    \centering
    \includegraphics[width=0.9\textwidth]{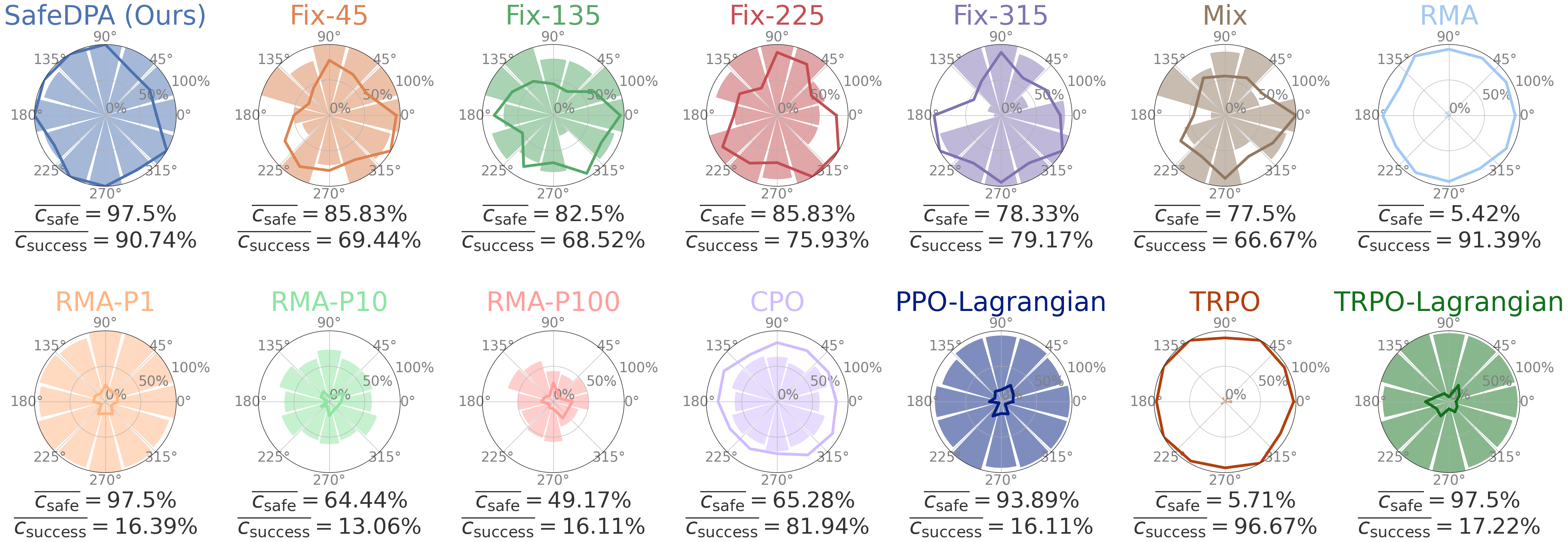} 
    \vspace{-5pt}
    \caption{The radius of the circles indicates task performance (success rate), while the size of the colored area represents the safety rate. \method stands out with the highest safety rate $97.5\%$. Though PPO and TRPO have marginally higher success rates, they frequently violate safety constraints.}
    \label{fig:EXP-SAFETY-GYM}
\end{figure*}

\subsection{Inverted Pendulum}
\label{SECTION:EXPERIMENT-PEND}
We adopt an inverted pendulum simulation augmented with external wind perturbations from~\cite{shi2021meta}. 
Wind (with fixed magnitude, and variable angle from 0 to 360 degrees) is introduced as the variable environment configuration $e_t$. In this setting, we aim to evaluate the adaptive safety of \method, and its dual capability in safety and task performance.

\subsubsection{Address \textbf{Q1}}
First, we consider the \textit{pendulum balancing} task, where the pendulum aims to sustain upright (angle $\theta=0$). Training the RL agent with the reward function \( r=-0.1\theta^2-0.1\dot\theta^2-0.0001u^2 \), where \( u \) is the torque, inherently aligns it to remain upright, fulfilling the safety criteria $|\theta|<45^\circ$. 
We compare \method with two baselines: $\text{Fix-}\alpha$ and Mix. For $\text{Fix-}\alpha$, dynamics models are trained in an environment with wind consistently directed as per $\alpha$, which can be one of $\{45,135,225,315\}$. Conversely, the Mix model is trained in various wind conditions without any adaptation (which can be viewed as a domain randomization method). 
To conduct a fair comparison of safety, we employ a control policy generating random actions for our adaptive safe module and other baseline methods.
We let \method, \Fix, and Mix share the same safety filter structure, and test them across environments with wind directions spanning from $0^{\circ}$ to $360^{\circ}$. The result is demonstrated in \Cref{FIG:Result-Pend-Safe}, in which the size of the colored area denotes the safety rate in different directions. 
Note that \Fix and Mix can only guarantee safety under specific wind directions, while \method achieves zero violation in all wind directions, outperforming baselines significantly.

\subsubsection{Address \textbf{Q2}}
To evaluate both task performance and safety simultaneously, we introduce the \textit{InvPendulum-Hazard} task, as shown in \Cref{FIG:Result-Pend-hazard}. In this setup, the RL objective is to maintain the pendulum in an upright position. Distinct from this goal, there is a safety constraint to avoid a hazard (represented by the red region shown in the pendulum diagram, as shown in \Cref{FIG:Result-Pend-hazard}). Notably, the pendulum's starting position is just beneath this hazard. To maintain an upright position without entering the hazard, the control policy should direct the pendulum to rotate counter-clockwise, through a longer route instead of the shorter clockwise path.
In addition to \Fix and Mix, our comparison includes RMA~\cite{kumar2021rma} and \RMAP. The \RMAP is an invariant of RMA that is trained with an extra reward penalty quantified by $\beta$, for any safety violation during training. 

As shown in \Cref{FIG:Result-Pend-hazard}, RMA consistently lacks safety guarantees, and the \RMAP methods are either too aggressive or too conservative. Meanwhile, both \Fix and Mix exhibit compromised safety across diverse $e_t$ environments. In contrast, \method consistently achieves 100\% success rate with 100\% safety rate, demonstrating its capability to ensure safety while achieving high performance across varying environment configurations.

\begin{table}[h]
\caption{The success rate and safety rate of \method, \method without fine-tuning and RMA with penalty on RC Car.}
\centering
\begin{tabular}{l|ccc|ccc}
\toprule
\multirow{2}{*}{Task} & \multicolumn{3}{c|}{Success Rate} & \multicolumn{3}{c}{Safety Rate} \\
 & Box & Chair & Car &  Box & Chair & Car \\
\midrule
\method    &     1.00 &   0.80 &        1.00 &  \textbf{0.90} &  \textbf{1.00} &   \textbf{0.75} \\
\method w/o Tuning &     0.80 &   1.00 &        1.00 &  0.70 &  0.45 &    0.34 \\
RMA-P    &     1.00 &   1.00 &        1.00 &  0.50 &  0.25 &       0.25 \\
\bottomrule
\end{tabular}
\label{TABLE:Real-Car}
\end{table}

\subsection{Safety Gym Simulation Experiments}
\label{SECTION:EXPERIMENT-SAFETY-GYM}

To answer \textbf{Q2}, we implement \method on the Safety Gym~\cite{Ray2019} benchmark. The \textit{point} robot in Safety Gym has higher state dimensions $x_t \in \R^{17}$, compared to pendulum's $x_t \in \R^2$. We introduce an external force $F_t \in \R^2$ (on the 2D plane) within the MuJoCo~\cite{todorov2012mujoco}, and define $F_t$ as environment configuration. The control task is to navigate the \textit{point} robot to a target and avoid hazards. For this setting, we also incorporate standard RL methods (PPO~\cite{schulman2017proximal} and TRPO~\cite{schulman2015trust}) and end-to-end safe RL methods including CPO~\cite{achiam2017constrained}, PPO-Lagrangian, and TRPO-Lagrangian~\cite{tessler2018reward} as benchmarks. Similar to experiments in \Cref{SECTION:EXPERIMENT-PEND}, we evaluate the algorithms across various environments with the direction of $F_t$ ranging from $0$ to $360$ degrees. The results are demonstrated in \Cref{fig:EXP-SAFETY-GYM}. Remarkably, our proposed \method is the sole algorithm that maintains both high performance and safety across all $F_t$ values. In contrast, other baseline algorithms tend to excel in either performance or safety, but not in both aspects simultaneously. 

To further illustrate that \method is able to give precise dynamics prediction which leads to more precise CBF value prediction, we present a visualization of the CBF landscape in \Cref{fig:HEATMAP}. 
We visualize the difference of CBF values $\Delta h(x_t, a_t) = h(f_{\theta_f} (x_t, \hat{z_t}) + g_{\theta_g} (x_t, \hat{z_t})a_t) - h(x_t)$ of different methods (Fix-45, Fix-135, Fix-225, Fix-315, Mix, and \method) with external force in the $135^{\circ}$ direction, and use the visualization of Fix-135 as the ground truth for reference because Fix-135 is trained and tested in the same external force.
As shown in \Cref{fig:HEATMAP}, the heatmap of our \method closely aligns with Fix-135 while other methods completely mismatch with Fix-135. This result proves the successful adaptation of dynamics prediction and CBF prediction of \method which further leads to successful adaptation of safe control. 
For instance, in the scenario depicted in \Cref{fig:HEATMAP}, a light color signifies safety, while a dark color signifies a lack of safety. When the throttle is set to 1, turning the car to the left results in a safer maneuver compared to turning it to the right, as demonstrated in \Cref{fig:HEATMAP} (left). Among the evaluated methods, only \method and Fix-135 provide accurate predictions in this scenario, leading to safer and less conservative actions. Incorrect estimations could result in detours or collisions with the obstacle (blue circle).

\begin{figure}[htbp]
    \centering
    \includegraphics[width=\columnwidth]{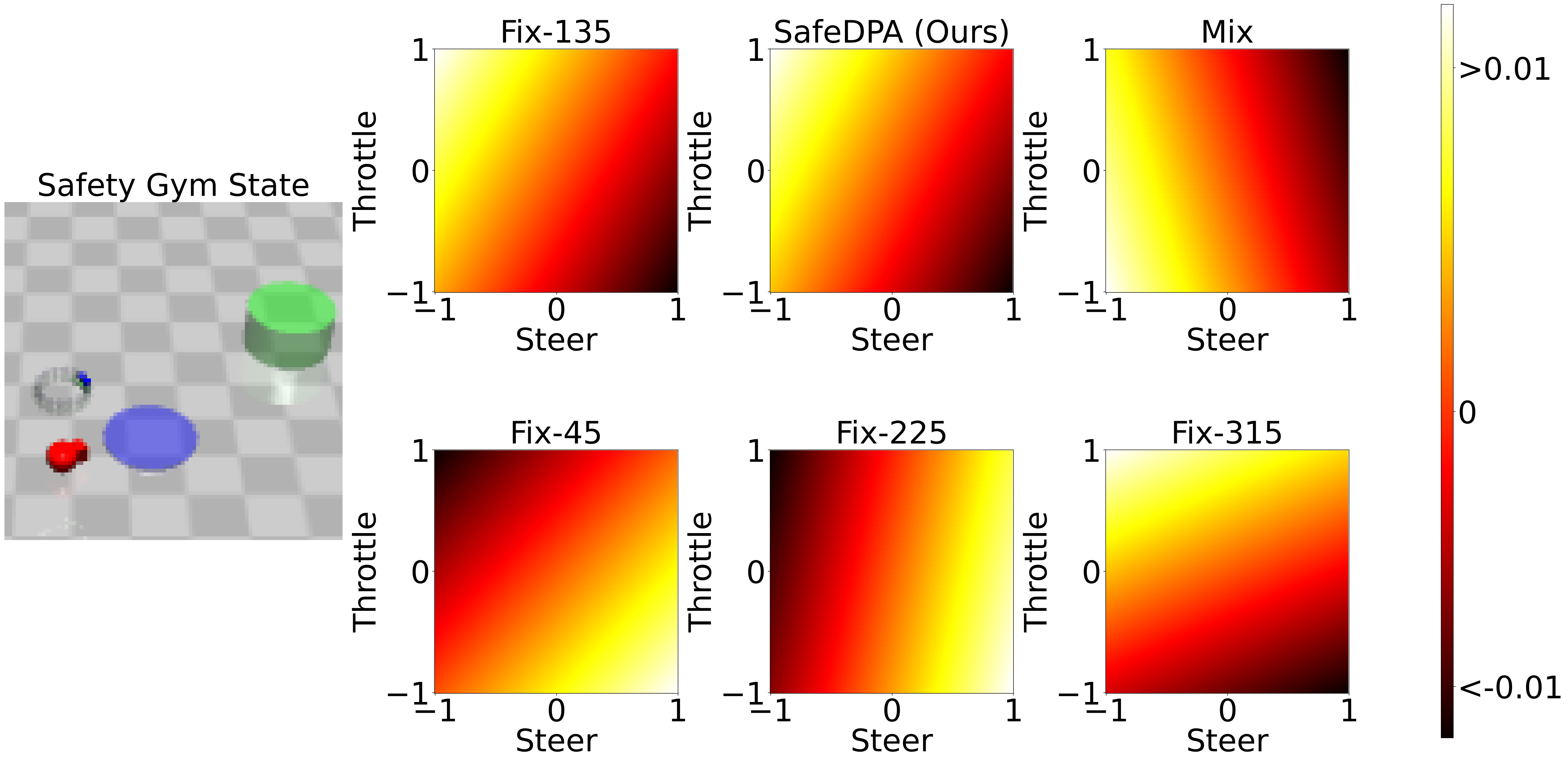} 
    \caption{Heat maps of the difference of CBF values $\Delta h(x_t, a_t) = h(f_{\theta_f} (x_t, \hat{z_t}) + g_{\theta_g} (x_t, \hat{z_t})a_t) - h(x_t)$ with external force in the $135^{\circ}$ direction. The x-axis represents the control space of turning actuator and the y-axis represents the control space of moving actuator. The tested environment is applied with external force in the $135^{\circ}$ direction, the alignment of heat maps between Fix-135 and \method proves the great adaption ability of \method.}
    \label{fig:HEATMAP}
\end{figure}

\subsection{Real-World Deployment}
\label{SECTION:EXPERIMENT-RC-CAR}
After successfully addressing \textbf{Q1} and \textbf{Q2} in simulations, we take a step further into the complexities of real-world experiments. In this section, we delve into evaluations of \method on an RC Car, aiming to tackle \textbf{Q3}.

\subsubsection{Experimental Setup}
We equip a 1/10 scale racing car with a Jetson TX2 to execute an on-board control policy. This setup ensures our controller runs efficiently at a frequency of 20Hz. To achieve precise control, we employ the Vicon motion capture system, which continuously provides information about the car's location. With this real-time data, the control objective for the car is to navigate towards a set goal, avoiding any collisions with a centrally positioned obstacle.

\subsubsection{Training in Simulation}
We start by constructing a simulator with a bicycle model. The bicycle model is a fusion of a kinematic model and a PID controller. We measure the car's physical properties to align the simulation. Within this simulation, we also simulate a random drag ($F_t\in \R^2$), which serves as an environment configuration $e_t$. The control input of the car is target velocity and steer. The state of the car consists of position, yaw, and real velocity. 
As for baselines, we train RMA with a reward penalty for collision (denoted as RMA-P). We also include \method without fine-tuning as baselines. Both \method and \method without fine-tuning utilize $2 \times 10^7$ state-action pairs in model learning.

\subsubsection{Real-world Fine-Tuning}
We use hard-coded commands to make the car move in circles on the unobstructed ground without the drag disturbance, and collect a small-scale dataset with data size merely amounting to $0.1\%$ of simulated data. 
We then use this dataset to fine-tune the adaption module and dynamics model of \method. 

\subsubsection{Adaption Deployment and Configurations}
The adaptation module is deployed synchronously, given that computations could be completed within the 0.05$s$ window. For real-world adaptation demonstrations, we set the obstacle in the same location as in the simulation, and test four distinct environment configurations: 1) \textbf{Car}: Standard RC car without dragging; 2) \textbf{Box}: Car attaching a paperboard box; 3) \textbf{Chair}: Car attaching a movable chair; 4) \textbf{Large Chair}: Car attaching a heavier chair. It's worth noting that configurations \textbf{Box, (Large) Chair} are totally unseen before deployment. These configurations test the model's zero-shot adaptation capabilities since these specific configurations were neither modeled in the simulation nor used during the fine-tuning phase. 

\subsubsection{Results}
The results are summarized in \Cref{TABLE:Real-Car} and \Cref{fig:RC-Car-Demo}. Even though \method has never seen chairs or boxes in simulation or fine-tuning, it still achieves a high safety rate. \method $3\times$ outperforms RMA-P on average in terms of safety. Even though RMA has good performance, it frequently violates safety constraints, which can be catastrophic in safety-critical systems. 
To further demonstrate the efficacy of the fine-tuning procedure (\textbf{Phase 3}), we apply both the dynamics model with fine-tuning (denoted as \textbf{Tuning}) and the dynamics model without fine-tuning (denoted as \textbf{w/o Tuning}) to predict trajectories of the car with the first 10 steps of a collected real-world trajectory. In \Cref{fig:ABLATION-trajectory}, the trajectory predicted by \textbf{Tuning} more closely resembles the real trajectory, whereas the trajectory predicted by \textbf{w/o Tuning} overestimates the moving distance. This overestimation may be caused by inaccurate estimations of physical parameters such as friction and inertia. \textbf{Tuning} mitigates this problem during fine-tuning. We also evaluate the empirical error of dynamics models on a dataset of real-world trajectories, which is mixed with \textbf{Car}, \textbf{Box}, and \textbf{Chair} data. The results in \Cref{fig:ABLATION-empirical-error} show that \textbf{Tuning} is approximately $10\times$ more accurate than \textbf{w/o Tuning} in estimating the future states.
Therefore, we could clearly answer \textbf{Q3} by the fact that fine-tuning with real-world data boosts the \method's generalizability and performance.

\begin{figure}[htp]
\centering
  \begin{subfigure}{.5\columnwidth}
    \includegraphics[width=\columnwidth]{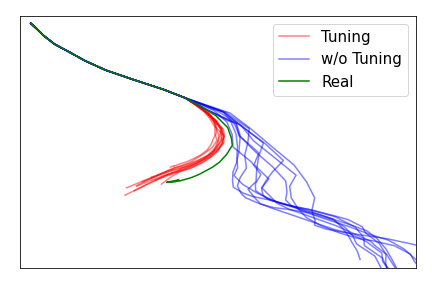}
    \caption{Predictive Trajectory}
    \label{fig:ABLATION-trajectory}
  \end{subfigure}%
  \begin{subfigure}{.5\columnwidth}
    \includegraphics[width=\columnwidth]{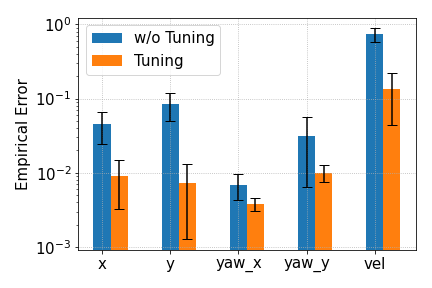}
    \caption{Empirical Error}
    \label{fig:ABLATION-empirical-error}
  \end{subfigure}%
\caption{In (a), we visualize the real-world trajectory (\textbf{Real} marked in green), the trajectory of the dynamics model with fine-tuning (\textbf{Tuning} marked in red), and the trajectory of the dynamics model without fine-tuning (\textbf{w/o Tuning} marked in blue). The predictive trajectories (\textbf{Tuning} and \textbf{w/o Tuning}) are generated by dynamics models after reading the first 10 steps and taking the same actions as the real trajectory. The fine-tuned dynamics model generates a trajectory that more closely resembles the real trajectory. In (b), we demonstrate that the tuned dynamics model is nearly $10\times$ more precise than the dynamics model without fine-tuning.}
\end{figure}

\section{Conclusion and Future Works}
In this paper, we propose \method, a novel framework that tackles the problem of policy adaptation while ensuring safety.
In the future, there are two interesting research directions. One is building closed-loop adaptation methods that directly adapt the policy network based on its performance and safety. The other is to generalize \method~to vision-based control systems and non-control-affine systems.

\section*{ACKNOWLEDGMENT}
We gratefully acknowledge the assistance of Dvij Kalaria in setting up the hardware.

\printbibliography

\end{document}